\documentclass[english]{article}
\usepackage[T1]{fontenc}
\usepackage[latin9]{inputenc}
\usepackage{geometry}
\usepackage{babel}
\usepackage{amsmath}
\usepackage{amsthm}
\usepackage{amssymb}
\usepackage{graphicx}
\usepackage{authblk}
\usepackage[numbers]{natbib}
\usepackage[unicode=true,
 bookmarks=false,
 breaklinks=false,pdfborder={0 0 1},backref=false,colorlinks=false]
 {hyperref}

\makeatletter

\theoremstyle{plain}
\newtheorem{thm}{\protect\theoremname}
  \theoremstyle{definition}
  \newtheorem{condition}{\protect\conditionname}
  \theoremstyle{plain}
  \newtheorem{prop}{\protect\propositionname}
  \theoremstyle{plain}
  \newtheorem{lem}{\protect\lemmaname}

\usepackage{babel}
\date{}

\makeatother

  \providecommand{\conditionname}{Condition}
  \providecommand{\lemmaname}{Lemma}
  \providecommand{\propositionname}{Proposition}
\providecommand{\theoremname}{Theorem}

\begin{document}
\global\long\def\argmin{\operatornamewithlimits{argmin}}
 \global\long\def\argmax{\operatornamewithlimits{argmax}}

\title{Interaction Screening: Efficient and Sample-Optimal Learning of Ising
Models}

\author[1]{\normalsize\textbf{Marc Vuffray}}
\author[2]{\textbf{Sidhant Misra}}
\author[1,3]{\textbf{Andrey Y. Lokhov}}
\author[1,3,4]{\textbf{Michael Chertkov}}
\affil[1]{Theoretical Division T-4, Los Alamos National Laboratory, Los Alamos, NM 87545, USA}
\affil[2]{Theoretical Division T-5, Los Alamos National Laboratory, Los Alamos, NM 87545, USA}
\affil[3]{Center for Nonlinear Studies, Los Alamos National Laboratory, Los Alamos, NM 87545, USA}
\affil[4]{Skolkovo Institute of Science and Technology, 143026 Moscow, Russia}
\affil[ ]{}
\affil[ ]{\texttt {\{vuffray, sidhant, lokhov, chertkov\}@lanl.gov}}

\maketitle

\begin{abstract}
We consider the problem of learning the underlying graph of an unknown
Ising model on $p$ spins from a collection of i.i.d. samples generated
from the model. We suggest a new estimator that is computationally
efficient and requires a number of samples that is near-optimal with
respect to previously established information-theoretic lower-bound.
Our statistical estimator has a physical interpretation in terms of
``interaction screening''. The estimator is consistent and is efficiently
implemented using convex optimization. We prove that with appropriate
regularization, the estimator recovers the underlying graph using
a number of samples that is logarithmic in the system size $p$ and
exponential in the maximum coupling-intensity and maximum node-degree. 
\end{abstract}

\section{Introduction\label{sec:Introduction}}

A Graphical Model (GM) describes a probability distribution over a
set of random variables which factorizes over the edges of a graph.
It is of interest to recover the structure of GMs from random samples.
The graphical structure contains valuable information on the dependencies
between the random variables. In fact, the neighborhood of a random
variable is the minimal set that provides us maximum information about
this variable. Unsurprisingly, GM reconstruction plays an important
role in various fields such as the study of gene expression \citep{MarbachCostelloKuffnerEtAl2012},
protein interactions \citep{MorcosPagnaniLuntEtAl2011}, neuroscience
\citep{SchneidmanBerrySegevEtAl2006}, image processing \citep{RothBlack2005},
sociology \citep{EaglePentlandLazer2009} and even grid science \citep{HeZhang2011,15DBS}.

The origin of the GM reconstruction problem is traced back to the
seminal 1968 paper by Chow and Liu \citep{ChowLiu1968}, where the
problem was posed and resolved for the special case of tree-structured
GMs. In this special tree case the maximum likelihood estimator is
tractable and is tantamount to finding a maximum weighted spanning-tree.
However, it is also known that in the case of general graphs with
cycles, maximum likelihood estimators are intractable as they require
computation of the partition function of the underlying GM, with notable
exceptions of the Gaussian GM, see for instance \citep{BanerjeeGhaoui2008},
and some other special cases, like planar Ising models without magnetic
field \citep{15JONC}.

A lot of efforts in this field has focused on learning Ising models,
which are the most general GMs over binary variables with pairwise
interaction/factorization. Early attempts to learn the Ising model
structure efficiently were heuristic, based on various mean-field
approximations, e.g. utilizing empirical correlation matrices \citep{Tanaka1998,KappenRodriguez1998,RoudiTyrchaHertz2009,Ricci-Tersenghi2012}.
These methods were satisfactory in cases when correlations decrease
with graph distance. However it was also noticed that the mean-field
methods perform poorly for the Ising models with long-range correlations.
This observation is not surprising in light of recent results stating
that learning the structure of Ising models using only their correlation
matrix is, in general, computationally intractable \citep{BreslerGamarnikShah2014,Montanari2015}.

Among methods that do not rely solely on correlation matrices but
take advantage of higher-order correlations that can be estimated
from samples, we mention the approach based on sparsistency of the
so-called regularized pseudo-likelihood estimator \citep{RavikumarWainwrightLafferty2010}.
This estimator, like the one we propose in this paper, is from the
class of M-estimators i.e. estimators that are the minimum of a sum
of functions over the sampled data \citep{NegahbanRavikumarWainwrightEtAl2012}.
The regularized pseudo-likelihood estimator is regarded as a surrogate
for the intractable likelihood estimator with an additive $\ell_{1}$-norm
penalty to encourage sparsity of the reconstructed graph. The sparsistency-based
estimator offers guarantees for the structure reconstruction, but
the result only applies to GMs that satisfy a certain condition that
is rather restrictive and hard to verify. It was also proven that
the sparsity pattern of the regularized pseudo-likelihood estimator
fails to reconstruct the structure of graphs with long-range correlations,
even for simple test cases \citep{MontanariPereira2009}.

Principal tractability of structure reconstruction of an arbitrary
Ising model from samples was proven only very recently. Bresler, Mossel
and Sly in \citep{BreslerMosselSly2013} suggested an algorithm which
reconstructs the graph without errors in polynomial time. They showed
that the algorithm requires number of samples that is logarithmic
in the number of variables. Although this algorithm is of a polynomial
complexity, it relies on an exhaustive neighborhood search, and the
degree of the polynomial is equal to the maximal node degree.

Prior to the work reported in this manuscript the best known procedure
for perfect reconstruction of an Ising model was through a greedy
algorithm proposed by Bresler in \citep{Bresler2015}. Bresler's algorithm
is based on the observation that the mutual information between neighboring
nodes in an Ising model is lower bounded. This observation allows
to reconstruct the Ising graph perfectly with only a logarithmic number
of samples and in time quasi-quadratic in the number of variables.
On the other hand, Bresler's algorithm suffers from two major practical
limitations. First, the number of samples, hence the running time
as well, scales double exponentially with respect to the largest node
degree and with respect to the largest coupling intensity between
pairs of variables. This scaling is rather far from the information-theoretic
lower-bound reported in \citep{SanthanamWainwright2012} predicting
instead a single exponential dependency on the two aforementioned
quantities. Second, Bresler's algorithm requires prior information
on the maximum and minimum coupling intensities as well as on the
maximum node degree, guarantees which, in reality, are not necessarily
available.

In this paper we propose a novel estimator for the graph structure
of an arbitrary Ising model which achieves perfect reconstruction
in quasi-quartic time (although we believe it can be provably reduced
to quasi-quadratic time) and with a number of samples logarithmic
in the system size. The algorithm is near-optimal in the sense that
the number of samples required to achieve perfect reconstruction,
and the run time, scale exponentially with respect to the maximum
node-degree and the maximum coupling intensity, thus matching parametrically
the information-theoretic lower bound of \citep{SanthanamWainwright2012}.
Our statistical estimator has the structure of a consistent M-estimator
implemented via convex optimization with an additional thresholding
procedure. Moreover it allows intuitive interpretation in terms of
what we coin the ``interaction screening''. We show that with a
proper $\ell_{1}$-regularization our estimator reconstructs couplings
of an Ising model from a number of samples that is near-optimal. In
addition, our estimator does not rely on prior information on the
model characteristics, such as maximum coupling intensity and maximum
degree.

The rest of the paper is organized as follows. In Section \ref{sec:Main-Results}
we give a precise definition of the structure estimation problem for
the Ising models and we describe in detail our method for structure
reconstruction within the family of Ising models. The main results
related to the reconstruction guarantees are provided by Theorem \ref{th:mr_SE_RISE}
and Theorem \ref{th:Struct_RISE}. In Section \ref{sec:Analysis}
we explain the strategy and the sequence of steps that we use to prove
our main theorems. Proofs of Theorem \ref{th:mr_SE_RISE} and Theorem
\ref{th:Struct_RISE} are summarized at the end of this Section. Section
\ref{sec:Numerical_Simulations} illustrates performance of our reconstruction
algorithm via simulations. Here we show on a number of test cases
that the sample complexity of the suggested method scales logarithmically
with the number of variables and exponentially with the maximum coupling
intensity. In Section \ref{sec:Conclusion} we discuss possible generalizations
of the algorithm and future work.

\section{Main Results\label{sec:Main-Results}}

Consider a graph $G=\left(V,E\right)$ with $p$ vertexes where $V=\left\{ 1,\dots,p\right\} $
is the vertex set and $E\subset V\times V$ is the undirected edge
set. Vertexes $i\in V$ are associated with binary random variables
$\sigma_{i}\in\left\{ -1,+1\right\} $ that are called spins. Edges
$\left(i,j\right)\in E$ are associated with non-zero real parameters
$\theta_{ij}^{*}\neq0$ that are called couplings. An Ising model
is a probability distribution $\mu$ over spin configurations $\underline{\sigma}=\left\{ \sigma_{1},\dots,\sigma_{p}\right\} $
that reads as follows: 
\begin{equation}
\mu\left(\underline{\sigma}\right)=\frac{1}{Z}\exp\left(\sum_{\left(i,j\right)\in E}\theta_{ij}^{*}\sigma_{i}\sigma_{j}\right),\label{eq:mr_gibbs_measure}
\end{equation}
where $Z$ is a normalization factor called the partition function.
\begin{equation}
Z=\sum_{\underline{\sigma}}\exp\left(\sum_{\left(i,j\right)\in E}\theta_{ij}^{*}\sigma_{i}\sigma_{j}\right).\label{eq:mr_partition_function}
\end{equation}
Notice that even though the main innovation of this paper \textendash{}
the efficient ``interaction screening'' estimator \textendash{}
can be constructed for the most general Ising models, we restrict
our attention in this paper to the special case of the Ising models
with zero local magnetic-field. This simplification is not necessary
and is done solely to simplify (generally rather bulky) algebra. Later
in the text we will thus refer to the zero magnetic field model (\ref{eq:mr_partition_function})
simply as the Ising model.

\subsection{Structure-Learning of Ising Models}

Suppose that $n$ sequences/samples of $p$ spins $\left\{ \underline{\sigma}^{\left(k\right)}\right\} _{k=1,\dots,n}$
are observed. Let us assume that each observed spin configuration
$\underline{\sigma}^{\left(k\right)}=\{\sigma_{1}^{(k)},\dots,\sigma_{p}^{(k)}\}$
is i.i.d. from (\ref{eq:mr_gibbs_measure}). Based on these measurements/samples
we aim to construct an estimator $\widehat{E}$ of the edge set that
reconstructs the structure exactly with high probability, i.e. 
\begin{equation}
\mathbb{P}\left[\widehat{E}=E\right]=1-\epsilon,\label{exact_reconstruction}
\end{equation}
where $\epsilon\in\left(0,\frac{1}{2}\right)$ is a prescribed reconstruction
error.

We are interested to learn structures of Ising models in the high-dimensional
regime where the number of observations/samples is of the order $n=\mathcal{O}\left(\ln p\right)$.
A necessary condition on the number of samples is given in \citep[Thm. 1]{SanthanamWainwright2012}.
This condition depends explicitly on the smallest and largest coupling
intensity 
\begin{equation}
\alpha:=\min_{\left(i,j\right)\in E}\lvert\theta_{ij}^{^{*}}\rvert,\,\beta:=\max_{\left(i,j\right)\in E}\lvert\theta_{ij}^{^{*}}\rvert,
\end{equation}
and on the maximal node degree 
\begin{equation}
d:=\max_{i\in V}\left|\partial i\right|,
\end{equation}
where the set of neighbors of a node $i\in V$ is denoted by $\partial i:=\left\{ j\mid\left(i,j\right)\in E\right\} $.

According to \citep{SanthanamWainwright2012}, in order to reconstruct
the structure of the Ising model with minimum coupling intensity $\alpha$,
maximum coupling intensity $\beta$, and maximum degree $d$, the
required number of samples should be at least 
\begin{equation}
n\geq\max\left(\frac{e^{\beta d}\ln\left(\frac{pd}{4}-1\right)}{4d\alpha e^{\alpha}},\frac{\ln p}{2\alpha\tanh\alpha}\right).\label{eq:mr_sample_low_bound}
\end{equation}
We see from Eq.~(\ref{eq:mr_sample_low_bound}) that the exponential
dependence on the degree and the maximum coupling intensity are both
unavoidable. Moreover, when the minimal coupling is small, the number
of samples should scale at least as $\alpha^{-2}$.

It remains unknown if the inequality (\ref{eq:mr_sample_low_bound})
is achievable. It is shown in \citep[Thm. 3]{SanthanamWainwright2012}
that there exists a reconstruction algorithm with error probability
$\epsilon\in\left(0,\frac{1}{2}\right)$ if the number of samples
is greater than 
\begin{equation}
n\geq\left(\frac{\beta d\left(3e^{2\beta d}+1\right)}{\sinh^{2}\left(\alpha/4\right)}\right)^{2}\left(16\log p+4\ln\left(2/\epsilon\right)\right).\label{eq:mr_samples_up_bound}
\end{equation}
Unfortunately, the existence proof presented in \citep{SanthanamWainwright2012}
is based on an exhaustive search with the intractable maximum likelihood
estimator and thus it does not guarantee actual existence of an algorithm
with low computational complexity. Notice also that the number of
samples in (\ref{eq:mr_samples_up_bound}) scales as $\exp\left(4\beta d\right)$
when $d$ and $\beta$ are asymptotically large and as $\alpha^{-4}$
when $\alpha$ is asymptotically small.

\subsection{Regularized Interaction Screening Estimator}

The main contribution of this paper consists in presenting explicitly
a structure-learning algorithm that is of low complexity and which
is near-optimal with respect to bounds (\ref{eq:mr_sample_low_bound})
and (\ref{eq:mr_samples_up_bound}). Our algorithm reconstructs the
structure of the Ising model exactly, as stated in Eq.~(\ref{exact_reconstruction}),
with an error probability $\epsilon\in\left(0,\frac{1}{2}\right)$,
and with a number of samples which is at most proportional to $\exp\left(6\beta d\right)$
and $\alpha^{-2}$. (See Theorem \ref{th:mr_SE_RISE} and Theorem
\ref{th:Struct_RISE} below for mathematically accurate statements.)
Our algorithm consists of two steps. First, we estimate couplings
in the vicinity of every node. Then, on the second step, we threshold
the estimated couplings that are sufficiently small to zero. Resulting
zero coupling means that the corresponding edge is not present.

Denote the set of couplings around node $u\in V$ by the vector $\underline{\theta}_{u}^{*}\in\mathbb{R}^{p-1}$.
In this, slightly abusive notation, we use the convention that if
a coupling is equal to zero it reads as absence of the edge, i.e.
$\theta_{ui}^{*}=0$ if and only if $\left(u,i\right)\notin E$. Note
that if the node degree is bounded by $d$, it implies that the vector
of couplings $\underline{\theta}_{u}^{*}$ is non-zero in at most
$d$ entries.

Our estimator for couplings around node $u\in V$ is based on the
following loss function coined the Interaction Screening Objective
(ISO): 
\begin{equation}
\mathcal{S}_{n}\left(\underline{\theta}_{u}\right)=\frac{1}{n}\sum_{k=1}^{n}\exp\left(-\sum_{i\in V\setminus u}\theta_{ui}\sigma_{u}^{(k)}\sigma_{i}^{(k)}\right).\label{eq:mr_iso}
\end{equation}

The ISO is an empirical weighted-average and its gradient is the vector
of weighted pair-correlations involving $\sigma_{u}$. At $\underline{\theta}_{u}=\underline{\theta}_{u}^{*}$
the exponential weight cancels exactly with the corresponding factor
in the distribution \eqref{eq:mr_gibbs_measure}. As a result, weighted
pair-correlations involving $\sigma_{u}$ vanish as if $\sigma_{u}$
was uncorrelated with any other spins or completely ``screened''
from them, which explains our choice for the name of the loss function.
This remarkable ``screening'' feature of the ISO suggests the following
choice of the Regularized Interaction Screening Estimator (RISE) for
the interaction vector around node $u$: 
\begin{equation}
\underline{\widehat{\theta}}_{u}\left(\lambda\right)=\argmin_{\underline{\theta}_{u}\in\mathbb{R}^{p-1}}\mathcal{S}_{n}\left(\underline{\theta}_{u}\right)+\lambda\left\Vert \underline{\theta}_{u}\right\Vert _{1},\label{eq:mr_RISE}
\end{equation}
where $\lambda>0$ is a tunable parameter promoting sparsity through
the additive $\ell_{1}$-penalty. Notice that the ISO is the empirical
average of an exponential function of $\underline{\theta}_{u}$ which
implies it is convex. Moreover, addition of the $\ell_{1}$-penalty
preserves the convexity of the minimization objective in Eq.~(\ref{eq:mr_RISE}).

As expected, the performance of RISE does depend on the choice of
the penalty parameter $\lambda$. If $\lambda$ is too small $\underline{\widehat{\theta}}_{u}\left(\lambda\right)$
is too sensitive to statistical fluctuations. On the other hand, if
$\lambda$ is too large $\underline{\widehat{\theta}}_{u}\left(\lambda\right)$
has too much of a bias towards zero. In general, the optimal value
of $\lambda$ is hard to guess. Luckily, the following theorem provides
strong guarantees on the square error for the case when $\lambda$
is chosen to be sufficiently large. 
\begin{thm}[Square Error of RISE]
\label{th:mr_SE_RISE}Let $\left\{ \underline{\sigma}^{\left(k\right)}\right\} _{k=1,\dots,n}$
be $n$ realizations of $p$ spins drawn i.i.d. from an Ising model
with maximum degree $d$ and maximum coupling intensity $\beta$.
Then for any node $u\in V$ and for any $\epsilon_{1}>0$, the square
error of the Regularized Interaction Screening Estimator (\ref{eq:mr_RISE})
with penalty parameter $\lambda=4\sqrt{\frac{\ln(3p/\epsilon_{1})}{n}}$
is bounded with probability at least $1-\epsilon_{1}$ by 
\begin{equation}
\left\Vert \underline{\widehat{\theta}}_{u}\left(\lambda\right)-\underline{\theta}_{u}^{*}\right\Vert _{2}\leq2^{8}\sqrt{d}\left(d+1\right)e^{3\beta d}\sqrt{\frac{\ln\frac{3p}{\epsilon_{1}}}{n}},\label{eq:mr_SE_RISE_bound}
\end{equation}
whenever $n\geq2^{14}d^{2}\left(d+1\right)^{2}e^{6\beta d}\ln\frac{3p^{2}}{\epsilon_{1}}$. 
\end{thm}
Our structure estimator (for the second step of the algorithm), Structure-RISE,
takes RISE output and thresholds couplings whose absolute value is
less than $\alpha/2$ to zero: 
\begin{equation}
\widehat{E}\left(\lambda,\alpha\right)=\left\{ \left(i,j\right)\in V\times V\mid\widehat{\theta}_{ij}\left(\lambda\right)+\widehat{\theta}_{ji}\left(\lambda\right)\geq\alpha\right\} .\label{eq:mr_structure_rise}
\end{equation}
Performance of the Structure-RISE is fully quantified by the following
Theorem. 
\begin{thm}[Structure Learning of Ising Models]
\label{th:Struct_RISE}Let $\left\{ \underline{\sigma}^{\left(k\right)}\right\} _{k=1,\dots,n}$
be $n$ realizations of $p$ spins drawn i.i.d. from an Ising model
with maximum degree $d$, maximum coupling intensity $\beta$ and
minimal coupling intensity $\alpha$. Then for any $\epsilon_{2}>0$,
Structure-RISE with penalty parameter $\lambda=4\sqrt{\frac{\ln(3p^{2}/\epsilon_{2})}{n}}$
reconstructs the edge-set perfectly with probability 
\begin{equation}
\mathbb{P}\left(\widehat{E}\left(\lambda,\alpha\right)=E\right)\geq1-\epsilon_{2},
\end{equation}
whenever $n\geq\max\left(d/16,\alpha^{-2}\right)2^{18}d\left(d+1\right)^{2}e^{6\beta d}\ln\frac{3p^{3}}{\epsilon_{2}}$.
\end{thm}
Proofs of Theorem \ref{th:mr_SE_RISE} and Theorem \ref{th:Struct_RISE}
are given in Subsection \ref{subsec:Proof_Theorems}.

Theorem \ref{th:mr_SE_RISE} states that RISE recovers not only the
structure but also the correct value of the couplings up to an error
based on the available samples. It is possible to improve the square-error
bound (\ref{eq:mr_SE_RISE_bound}) even further by first, running
Structure-RISE to recover edges, and then re-running RISE with $\lambda=0$
for the remaining non-zero couplings.

The computational complexity of RISE is equal to the complexity of
minimizing the convex ISO and, as such, it scales at most as $\mathcal{O}\left(np^{3}\right)$.
Therefore, computational complexity of Structure-RISE scales at most
as $\mathcal{O}\left(np^{4}\right)$ simply because one has to call
RISE at every node. We believe that this running-time estimate can
be proven to be quasi-quadratic when using first-order minimization-techniques,
in the spirit of \citep{AgarwalNegahbanWainwright2010}. We have observed
through numerical experiments that such techniques implement Structure-RISE
with running-time $\mathcal{O}\left(np^{2}\right)$.

Notice that in order to implement RISE there is no need for prior
knowledge on the graph parameters. This is a considerable advantage
in practical applications where the maximum degree or bounds on couplings
are often unknown.

\section{Analysis\label{sec:Analysis}}

The Regularized Interaction Screening Estimator (\ref{eq:mr_RISE})
is from the class of the so-called regularized M-estimators. Negahban
et al. proposed in \citep{NegahbanRavikumarWainwrightEtAl2012} a
framework to analyze the square error of such estimators. As per \citep{NegahbanRavikumarWainwrightEtAl2012},
enforcing only two conditions on the loss function is sufficient to
get a handle on the square error of an $\ell_{1}$-regularized M-estimator.

The first condition links the choice of the penalty parameter to the
gradient of the objective function. 
\begin{condition}
\label{cond:gradient_penalty}The $\ell_{1}$-penalty parameter strongly
enforces regularization if it is greater than any partial derivatives
of the objective function at $\theta_{u}=\theta_{u}^{*}$, i.e. 
\begin{equation}
\lambda\geq2\left\Vert \nabla\mathcal{S}_{n}\left(\theta_{u}^{*}\right)\right\Vert _{\infty}.
\end{equation}
\end{condition}
Condition \ref{cond:gradient_penalty} guarantees that if the vector
of couplings $\theta_{u}^{*}$ has at most $d$ non-zero entries,
then the estimation difference $\underline{\widehat{\theta}}_{u}\left(\lambda\right)-\underline{\theta}_{u}^{*}$
lies within the set 
\begin{equation}
K:=\left\{ \Delta\in\mathbb{R}^{p-1}\mid\left\Vert \Delta\right\Vert _{1}\leq4\sqrt{d}\left\Vert \Delta\right\Vert _{2}\right\} .\label{eq:ana_set_pseudo_cone}
\end{equation}
The second condition ensure that the objective function is strongly
convex in a restricted subset of $\mathbb{R}^{p-1}$. Denote the reminder
of the first-order Taylor expansion of the objective function by 
\begin{equation}
\delta\mathcal{S}_{n}\left(\Delta_{u},\theta_{u}^{*}\right):=\mathcal{S}_{n}\left(\theta_{u}^{*}+\Delta_{u}\right)-\mathcal{S}_{n}\left(\theta_{u}^{*}\right)-\left\langle \nabla\mathcal{S}_{n}\left(\theta_{u}^{*}\right),\Delta_{u}\right\rangle ,\label{eq:ana_taylor_remainder}
\end{equation}
where $\Delta_{u}\in\mathbb{R}^{p-1}$ is an arbitrary vector. Then
the second condition reads as follows.
\begin{condition}
\label{cond:restricted_convexity}The objective function is restricted
strongly convex with respect to $K$ on a ball of radius $R$ centered
at $\theta_{u}=\theta_{u}^{*}$, if for all $\Delta_{u}\in K$ such
that $\left\Vert \Delta_{u}\right\Vert _{2}\leq R$, there exists
a constant $\kappa>0$ such that 
\begin{equation}
\delta\mathcal{S}_{n}\left(\Delta_{u},\theta_{u}^{*}\right)\geq\kappa\left\Vert \Delta_{u}\right\Vert _{2}^{2}.
\end{equation}

Strong regularization and restricted strong convexity enables us to
control that the minimizer $\underline{\widehat{\theta}}_{u}$ of
the full objective (\ref{eq:mr_RISE}) lies in the vicinity of the
sparse vector of parameters $\theta_{u}^{*}$. The precise formulation
is given in the proposition following from \citep[Thm. 1]{NegahbanRavikumarWainwrightEtAl2012}. 
\end{condition}
\begin{prop}
\label{prop:m_estimators_se}If the $\ell_{1}$-regularized M-estimator
of the form (\ref{eq:mr_RISE}) satisfies Condition \ref{cond:gradient_penalty}
and Condition \ref{cond:restricted_convexity} with $R>3\sqrt{d}\frac{\lambda}{\kappa}$
then the square-error is bounded by 
\begin{equation}
\left\Vert \underline{\widehat{\theta}}_{u}-\underline{\theta}_{u}^{*}\right\Vert _{2}\leq3\sqrt{d}\frac{\lambda}{\kappa}.
\end{equation}
\end{prop}

\subsection{Gradient Concentration\label{subsec:grad_concentration}}

Like the ISO (\ref{eq:mr_iso}), its gradient in any component $l\in V\setminus u$
is an empirical average 
\begin{equation}
\frac{\partial}{\partial\theta_{ul}}\mathcal{S}_{n}\left(\underline{\theta}_{u}\right)=\frac{1}{n}\sum_{k=1}^{n}X_{ul}^{(k)}\left(\underline{\theta}_{u}\right),\label{eq:ana_iso_gradient}
\end{equation}
where the random variables $X_{ul}^{(k)}\left(\underline{\theta}_{u}\right)$
are i.i.d and they are related to the spin configurations according
to 
\begin{equation}
X_{ul}\left(\underline{\theta}_{u}\right)=-\sigma_{u}\sigma_{l}\exp\left(-\sum_{i\in V\setminus u}\theta_{ui}\sigma_{u}\sigma_{i}\right).\label{eq:ana_gradient_decomposition}
\end{equation}
In order to prove that the ISO gradient concentrates we have to state
few properties of the support, the mean and the variance of the random
variables (\ref{eq:ana_gradient_decomposition}), expressed in the
following three Lemmas.

The first of the Lemmas states that at $\underline{\theta}_{u}=\underline{\theta}_{u}^{*}$,
the random variable $X_{ul}\left(\underline{\theta}_{u}^{*}\right)$
has zero mean. 
\begin{lem}
\label{lem:gradient_mean}For any Ising model with $p$ spins and
for all $l\neq u\in V$ 
\begin{equation}
\mathbb{E}\left[X_{ul}\left(\underline{\theta}_{u}^{*}\right)\right]=0.
\end{equation}
\end{lem}
\begin{proof}
By direct computation, we find that 
\begin{eqnarray}
\mathbb{E}\left[X_{ul}\left(\underline{\theta}_{u}^{*}\right)\right] & = & \mathbb{E}\left[-\sigma_{u}\sigma_{l}\exp\left(-\sum_{i\in\partial u}\theta_{ui}^{*}\sigma_{u}\sigma_{i}\right)\right]\nonumber \\
 & = & \frac{-1}{Z}\sum_{\underline{\sigma}}\sigma_{u}\sigma_{l}\exp\left(\sum_{\left(i,j\right)\in E}\theta_{ij}^{*}\sigma_{i}\sigma_{j}-\sum_{i\in\partial u}\theta_{ui}^{*}\sigma_{u}\sigma_{i}\right)=0,
\end{eqnarray}
where in the last line we use the fact that the exponential terms
involving $\sigma_{u}$ cancel, implying that the sum over $\sigma_{u}\in\left\{ -1,+1\right\} $
is zero. 
\end{proof}
As a direct corollary of the Lemma \ref{lem:gradient_mean}, $\underline{\theta}_{u}=\underline{\theta}_{u}^{*}$
is always a minimum of the averaged ISO (\ref{eq:mr_iso}).

The second Lemma proves that at $\underline{\theta}_{u}=\underline{\theta}_{u}^{*}$,
the random variable $X_{ul}\left(\underline{\theta}_{u}^{*}\right)$
has a variance equal to one. 
\begin{lem}
\label{lem:gradient_variance}For any Ising model with $p$ spins
and for all $l\neq u\in V$ 
\begin{equation}
\mathbb{E}\left[X_{ul}\left(\underline{\theta}_{u}^{*}\right)^{2}\right]=1.
\end{equation}
\end{lem}
\begin{proof}
As a result of direct evaluation one derives 
\begin{eqnarray}
\mathbb{E}\left[X_{ul}\left(\underline{\theta}_{u}^{*}\right)^{2}\right] & = & \mathbb{E}\left[\exp\left(-2\sum_{i\in\partial u}\theta_{ui}^{*}\sigma_{u}\sigma_{i}\right)\right]\nonumber \\
 & = & \frac{1}{Z}\sum_{\underline{\sigma}}\exp\left(\sum_{\left(i,j\right)\in E,i,j\neq u}\theta_{ij}^{*}\sigma_{i}\sigma_{j}-\sum_{i\in\partial u}\theta_{ui}^{*}\sigma_{u}\sigma_{i}\right)\nonumber \\
 & = & \frac{1}{Z}\sum_{\underline{\sigma}}\exp\left(\sum_{\left(i,j\right)\in E,i,j\neq u}\theta_{ij}^{*}\sigma_{i}\sigma_{j}+\sum_{i\in\partial u}\theta_{ui}^{*}\sigma_{u}\sigma_{i}\right)\nonumber \\
 & = & 1.
\end{eqnarray}
Notice that in the second line the first sum over edges (under the
exponential) does not depend on $\sigma_{u}$. Furthermore, the first
sum is invariant under the change of variables, $\sigma_{u}\rightarrow-\sigma_{u}$,
while the second sum changes sign. This transformation results in
appearance of the partition function in the numerator. 
\end{proof}
The next lemma states that at $\underline{\theta}_{u}=\underline{\theta}_{u}^{*}$,
the random variable $X_{ul}\left(\underline{\theta}_{u}^{*}\right)$
has a bounded support. 
\begin{lem}
\label{lem:gradient_support} For any Ising model with $p$ spins,
with maximum degree $d$ and maximum coupling intensity $\beta$,
it is guaranteed that for all $l\neq u\in V$ 
\begin{equation}
\left|X_{ul}\left(\underline{\theta}_{u}^{*}\right)\right|\leq\exp\left(\beta d\right).
\end{equation}
\end{lem}
\begin{proof}
Observe that components of $\underline{\theta}_{u}^{*}$ are smaller
than $\beta$ and at most $d$ of them are non-zero. Recall that spins
are binary, $\left\{ -1,+1\right\} $, which results in the following
estimate 
\begin{eqnarray}
\left|X_{ul}\left(\underline{\theta}_{u}^{*}\right)\right| & = & \left|-\sigma_{u}\sigma_{i}\exp\left(-\sum_{i\in\partial u}\theta_{ui}^{*}\sigma_{u}\sigma_{i}\right)\right|\nonumber \\
 & \leq & \exp\left(-\sum_{i\in\partial u}\theta_{ui}^{*}\sigma_{u}\sigma_{i}\right)\nonumber \\
 & \leq & \exp\left(\beta d\right).
\end{eqnarray}
\end{proof}
With Lemma \ref{lem:gradient_mean}, \ref{lem:gradient_variance}
and \ref{lem:gradient_support}, and using Berstein's inequality we
are now in position to prove that every partial derivative of the
ISO concentrates uniformly around zero as the number of samples grows. 
\begin{lem}
\label{lem:gradient_concentration} For any Ising model with $p$
spins, with maximum degree $d$ and maximum coupling intensity $\beta$.
For any $\epsilon_{3}>0$, if the number of observation satisfies
$n\geq\exp\left(2\beta d\right)\ln\frac{2p}{\epsilon_{3}}$, then
the following bound holds with probability at least $1-\epsilon_{3}$:
\begin{equation}
\left\Vert \nabla\mathcal{S}_{n}\left(\underline{\theta}_{u}^{*}\right)\right\Vert _{\infty}\leq2\sqrt{\frac{\ln\frac{2p}{\epsilon_{3}}}{n}}.\label{eq:gradient_concentration}
\end{equation}
\end{lem}
\begin{proof}
Let us first show that every term is individually bounded by the RHS
of \eqref{eq:gradient_concentration} with high-probability. We further
use the union bound to prove that all components are uniformly bounded
with high-probability. Utilizing Lemma \ref{lem:gradient_mean}, Lemma
\ref{lem:gradient_variance} and Lemma \ref{lem:gradient_support}
we apply the Bernstein's Inequality 
\begin{equation}
\mathbb{P}\left[\left|\frac{\partial}{\partial\theta_{ul}}\mathcal{S}_{n}\left(\underline{\theta}_{u}^{*}\right)\right|>t\right]\leq2\exp\left(-\frac{\frac{1}{2}t^{2}n}{1+\frac{1}{3}\exp\left(\beta d\right)t}\right).\label{eq:ana_berstein_inequality}
\end{equation}
Inverting the following relation 
\begin{equation}
s=\frac{\frac{1}{2}t^{2}n}{1+\frac{1}{3}\exp\left(\beta d\right)t},
\end{equation}
and substituting the result in the Eq.~(\ref{eq:ana_berstein_inequality})
one derives 
\begin{equation}
\mathbb{P}\left[\left|\frac{\partial}{\partial\theta_{ul}}\mathcal{S}_{n}\left(\underline{\theta}_{u}^{*}\right)\right|>\frac{1}{3}\left(u+\sqrt{\frac{18}{\exp\left(\beta d\right)}u+u^{2}}\right)\right]\leq2\exp\left(-s\right),\label{eq:ana_inverted_berstein_inquality}
\end{equation}
where $u=\frac{s}{n}\exp\left(\beta d\right).$

For $n\geq s\exp\left(2\beta d\right)$, we can simplify Eq.~(\ref{eq:ana_inverted_berstein_inquality})
to have an expression independent of $\beta$ and $d$ 
\begin{equation}
\mathbb{P}\left[\left|\frac{\partial}{\partial\theta_{ul}}\mathcal{S}_{n}\left(\underline{\theta}_{u}^{*}\right)\right|>2\sqrt{\frac{s}{n}}\right]\leq2\exp\left(-s\right).
\end{equation}
Using $s=\ln\frac{2p}{\epsilon_{3}}$ and the union bound on every
component of the gradient leads to the desired result. 
\end{proof}

\subsection{Restricted Strong-Convexity\label{subsec:restricted_convexity}}

The remainder of the first-order Taylor-expansion of the ISO, defined
in Eq.~(\ref{eq:ana_taylor_remainder}) is explicitly computed 
\begin{equation}
\delta\mathcal{S}_{n}\left(\Delta_{u},\theta^{*}\right)=\frac{1}{n}\sum_{k=1}^{n}\exp\left(-\sum_{i\in\partial u}\theta_{ui}^{*}\sigma_{u}^{(k)}\sigma_{i}^{(k)}\right)f\left(\sum_{i\in V\setminus u}\Delta_{ui}\sigma_{u}^{(k)}\sigma_{i}^{(k)}\right),\label{eq:ana_explicit_taylor}
\end{equation}
where the function $f\left(z\right)$ appearing in Eq.~\eqref{eq:ana_explicit_taylor}
reads 
\begin{equation}
f\left(z\right):=e^{-z}-1+z.\label{eq:ana_taylor_function}
\end{equation}
In the following lemma we prove that Eq.~(\ref{eq:ana_explicit_taylor})
is controlled by a much simpler expression using a lower-bound on
Eq.~\eqref{eq:ana_taylor_function}.
\begin{lem}
\label{lem:bound_exp_taylor} For all $\Delta_{u}\in\mathbb{R}^{p-1}$,
the remainder of the first-order Taylor expansion admits the following
lower-bound
\begin{align}
\delta\mathcal{S}_{n}\left(\Delta_{u},\theta^{*}\right) & \geq\frac{e^{-\beta d}}{2+\left\Vert \Delta_{u}\right\Vert _{1}}\Delta_{u}^{\top}H^{n}\Delta_{u}
\end{align}
where the matrix $H^{n}$ is an empirical covariance matrix with elements
$i,j\in V\setminus u$
\begin{equation}
H_{ij}^{n}=\frac{1}{n}\sum_{k=1}^{n}\sigma_{i}^{(k)}\sigma_{j}^{(k)}.
\end{equation}
\end{lem}
\begin{proof}
We start to prove a lower-bound on the function $f\left(z\right)$
valid for all $z\in\mathbb{R}$, 
\begin{equation}
f\left(z\right)\geq\frac{z^{2}}{2+\left|z\right|}.\label{eq:ana_bound_function_f}
\end{equation}
To see this, define an auxiliary function $g\left(z\right)$ as follows
\begin{align}
g\left(z\right): & =\left(2+\left|z\right|\right)f\left(z\right)-z^{2}\nonumber \\
 & =\left(2+\left|z\right|\right)\left(e^{-z}-1+z\right)-z^{2}.
\end{align}
We show that $g\left(z\right)$ achieves its minimum at $g\left(0\right)=0.$
Observe that the first derivative of $g\left(z\right)$ vanishes at
zero from both the negative and positive side
\begin{align}
\lim_{z\rightarrow0_{+}}\frac{d}{dz}g\left(z\right) & =\lim_{z\rightarrow0_{-}}\frac{d}{dz}g\left(z\right)=0.
\end{align}
Moreover for all $z>0$ the second derivative of $g\left(z\right)$
is non-negative
\begin{equation}
\frac{d^{2}}{dz^{2}}g\left(z\right)=ze^{-z}>0.
\end{equation}
A similar result holds for $z<0$
\begin{equation}
\frac{d^{2}}{dz^{2}}g\left(z\right)=4\left(e^{-z}-1\right)-ze^{-z}>0,
\end{equation}
proving that for all $z$, $g\left(z\right)\geq g\left(0\right)=0$. 

Combining Eq.~\eqref{eq:ana_bound_function_f} with the straightforward
inequalities
\begin{equation}
\left|\sum_{i\in V\setminus u}\Delta_{ui}\sigma_{u}^{(k)}\sigma_{i}^{(k)}\right|\leq\left\Vert \Delta_{u}\right\Vert _{1},
\end{equation}
and 
\begin{equation}
\exp\left(-\sum_{i\in\partial u}\theta_{ui}^{*}\sigma_{u}^{(k)}\sigma_{i}^{(k)}\right)\geq\exp\left(-\beta d\right),
\end{equation}
leads us to the following lower-bound on the remainder of the first-order
Taylor expansion of the ISO
\begin{align}
\delta\mathcal{S}_{n}\left(\Delta_{u},\theta^{*}\right) & \geq\frac{e^{-\beta d}}{2+\left\Vert \Delta_{u}\right\Vert _{1}}\frac{1}{n}\sum_{k=1}^{n}\left(\sum_{i\in V\setminus u}\Delta_{ui}\sigma_{u}^{(k)}\sigma_{i}^{(k)}\right)^{2}\nonumber \\
 & =\frac{e^{-\beta d}}{2+\left\Vert \Delta_{u}\right\Vert _{1}}\Delta_{u}^{\top}H^{n}\Delta_{u},
\end{align}
where in the last line we used the trivial identity $\sigma_{u}^{(k)}\cdot\sigma_{u}^{(k)}=1$.
\end{proof}
Lemma \ref{lem:bound_exp_taylor} enables us to control the randomness
in $\delta\mathcal{S}_{n}\left(\Delta_{u},\theta^{*}\right)$ through
the simpler matrix $H^{n}$ that is \emph{independent} of $\Delta_{u}$.
This last point is crucial as we show in the next lemma that $H^{n}$
concentrates independently of $\Delta_{u}$ towards its mean.
\begin{lem}
\label{lem:concentration_covariance_matrix} Consider an Ising model
with $p$ spins, with maximum degree $d$ and maximum coupling intensity
$\beta$. Let $\delta>0$, $\epsilon_{4}>0$ and $n\geq\frac{2}{\delta^{2}}\ln\frac{p^{2}}{\epsilon_{4}}$.
Then with probability greater than $1-\epsilon_{4}$, we have for
all $i,j\in V\setminus u$
\begin{equation}
\left|H_{ij}^{n}-H_{ij}\right|\leq\delta,
\end{equation}
where the matrix $H$ is the covariance matrix with elements \textup{$i,j\in V\setminus u$}
\begin{equation}
H_{ij}=\mathbb{E}\left[\sigma_{i}\sigma_{j}\right].
\end{equation}
\end{lem}
\begin{proof}
We recall that the matrix elements of the empirical covariance matrix
read
\begin{equation}
H_{ij}^{n}=\frac{1}{n}\sum_{k=1}^{n}\sigma_{i}^{(k)}\sigma_{j}^{(k)}.
\end{equation}
Since $\left|\sigma_{i}^{(k)}\sigma_{j}^{(k)}\right|\leq1$ using
Hoeffding's inequality, we have
\begin{equation}
\mathbb{P}\left[\left|H_{ij}^{n}-H_{ij}\right|\geq\delta\right]\leq2\exp\left(-\frac{n\delta^{2}}{2}\right).
\end{equation}
As $H_{ij}^{n}$ is symmetric we use the union bound over the elements
$i<j\in V\setminus u$ to get
\begin{equation}
\mathbb{P}\left[\left|H_{ij}^{n}-H_{ij}\right|\geq\delta\quad\forall i,j\in V\setminus u\right]\leq1-p^{2}\exp\left(-\frac{n\delta^{2}}{2}\right).
\end{equation}
\end{proof}
The last ingredient that we need is a proof that the smallest eigenvalue
of the covariance matrix $H$ is bounded away from zero \emph{independently
of the dimension }$p$. Equivalently the next lemma shows that the
quadratic form associated with $H$ is non-degenerate regardless of
the value of $p$.
\begin{lem}
\label{lem:quadratic_correlation_matrix} Consider an Ising model
with $p$ spins, with maximum degree $d$ and maximum coupling intensity
$\beta$. For all $\Delta_{u}\in\mathbb{R}^{p-1}$ the following bound
holds 
\begin{equation}
\Delta_{u}^{\top}H\Delta_{u}\geq\frac{e^{-2\beta d}}{d+1}\left\Vert \Delta_{u}\right\Vert _{2}^{2}.\label{eq:ana_exp_delta_square}
\end{equation}
\end{lem}
\begin{proof}
Our proof strategy here follows \citep[Cor. 3.1]{Montanari2015}.
Notice that the probability measure of the Ising model is symmetric
with respect to the sign flip, i.e. $\mu\left(\sigma_{1},\dots,\sigma_{p}\right)=\mu\left(-\sigma_{1},\dots,-\sigma_{p}\right)$.
Thus any spin has zero mean, which implies that for every $\Delta_{u}\in\mathbb{R}^{p-1}$
\begin{eqnarray}
\mathbb{E}\left[\left(\sum_{i\in V\setminus u}\Delta_{ui}\sigma_{i}\right)\right] & = & 0.
\end{eqnarray}
This allows to reinterpret the left-hand side of Eq.~(\ref{eq:ana_exp_delta_square})
as a variance, using that $\sigma_{u}^{2}=1,$ 
\begin{eqnarray}
\Delta_{u}^{\top}H\Delta_{u} & = & \sum_{i,j\in V\setminus u}\Delta_{ui}\mathbb{E}\left[\sigma_{i}\sigma_{j}\right]\Delta_{uj}\nonumber \\
 & = & \mathbb{E}\left[\left(\sum_{i\in V\setminus u}\Delta_{ui}\sigma_{i}\right)^{2}\right]\nonumber \\
 & = & \mathrm{Var}\left[\sum_{i\in V\setminus u}\Delta_{ui}\sigma_{i}\right].\label{eq:ana_var_interpretation}
\end{eqnarray}

Construct a subset $A\subset V$ recursively as follows: (i) let $i_{0}=\argmax_{j\in V\setminus u}\Delta_{uj}^{2}$
and define $A_{0}=\{i_{0}\}$, (ii) given $A_{t}=\{i_{0},\ldots,i_{t}\}$,
let $B_{t}=\left\{ j\in V\setminus A_{t}\mid\partial j\cap A_{t}=\emptyset\right\} $
and $i_{t+1}=\argmax_{j\in B_{t}\setminus u}\Delta_{uj}^{2}$ and
set $A_{t+1}=A_{t}\cup\{i_{t+1}\}$, (iii) terminate when $B_{t}\setminus u=\emptyset$
and declare $A=A_{t}$. 

The set $A$ possesses the following two main properties. First, every
node $i\in A$ does not have any neighbors in $A$ and, second, 
\begin{equation}
\left(d+1\right)\sum_{i\in A}\Delta_{ui}^{2}\geq\sum_{i\in V\setminus u}\Delta_{ui}^{2}.\label{eq:ana_delta_square_a}
\end{equation}

We apply the law of total variance to (\ref{eq:ana_var_interpretation})
by conditioning on the set of spins $\underline{\sigma}_{A^{c}}$
with indexes belonging to the complementary set $A^{c}$, 
\begin{eqnarray}
\mathrm{Var}\left[\sum_{i\in V\setminus u}\Delta_{ui}\sigma_{i}\right] & \geq & \mathbb{E}\left[\textrm{Var}\left[\left.\sum_{i\in V\setminus u}\Delta_{ui}\sigma_{i}\right|\underline{\sigma}_{A^{c}}\right]\right]\nonumber \\
 & = & \sum_{i\in A}\Delta_{ui}^{2}\mathbb{E}\left[\textrm{Var}\left[\sigma_{i}\mid\underline{\sigma}_{A^{c}}\right]\right],
\end{eqnarray}
where in the last line one uses that the spins in $A$ are conditionally
independent given their neighbors $\underline{\sigma}_{A^{c}}$. One
concludes the proof by using relation (\ref{eq:ana_delta_square_a})
and the fact that the conditional variance of a spin given its neighbors
is bounded from below: 
\begin{eqnarray}
\textrm{Var}\left[\sigma_{i}\mid\underline{\sigma}_{A^{c}}\right] & = & 1-\tanh^{2}\left(\sum_{j\in\partial i}\theta_{ij}^{*}\sigma_{j}\right)\nonumber \\
 & \geq & \exp\left(-2\beta d\right).
\end{eqnarray}
\end{proof}
We stress that Lemma \ref{lem:quadratic_correlation_matrix} is a
deterministic result valid for all $\Delta_{u}\in\mathbb{R}^{p-1}$.
We are now in position to prove the restricted strong convexity of
the ISO.
\begin{lem}
\label{lem:restricted_strong_convexity}Consider an Ising model with
$p$ spins, with maximum degree $d$ and maximum coupling intensity
$\beta$. For all $\epsilon_{4}>0$ and $R>0$, when $n\geq2^{11}d^{2}\left(d+1\right)^{2}e^{4\beta d}\ln\frac{p^{2}}{\epsilon_{4}}$
the ISO (\ref{eq:mr_iso}) satisfies, with probability at least $1-\epsilon_{4}$,
the restricted strong convexity condition
\begin{equation}
\delta\mathcal{S}_{n}\left(\Delta_{u},\theta_{u}^{*}\right)\geq\frac{e^{-3\beta d}}{4\left(d+1\right)\left(1+2\sqrt{d}R\right)}\left\Vert \Delta_{u}\right\Vert _{2}^{2},
\end{equation}
for all $\Delta_{u}\in\mathbb{R}^{p-1}$ such that \textup{$\left\Vert \Delta_{u}\right\Vert _{1}\leq4\sqrt{d}\left\Vert \Delta_{u}\right\Vert _{2}$}
and $\left\Vert \Delta_{u}\right\Vert _{2}\leq R$.
\end{lem}
\begin{proof}
First we apply Lemma~\ref{lem:bound_exp_taylor} to get the quadratic
bound
\begin{align}
\delta\mathcal{S}_{n}\left(\Delta_{u},\theta^{*}\right) & \geq\frac{e^{-\beta d}}{2+\left\Vert \Delta_{u}\right\Vert _{1}}\Delta_{u}^{\top}H^{n}\Delta_{u}\nonumber \\
 & \geq\frac{e^{-\beta d}}{2\left(1+2\sqrt{d}R\right)}\Delta_{u}^{\top}H^{n}\Delta_{u}.
\end{align}
Second we use Lemma~\ref{lem:quadratic_correlation_matrix} to bound
the quadratic form
\begin{align}
\Delta_{u}^{\top}H^{n}\Delta_{u} & =\Delta_{u}^{\top}H\Delta_{u}+\Delta_{u}^{\top}\left(H^{n}-H\right)\Delta_{u}\nonumber \\
 & \geq\frac{e^{-2\beta d}}{d+1}\left\Vert \Delta_{u}\right\Vert _{2}^{2}+\Delta_{u}^{\top}\left(H^{n}-H\right)\Delta_{u}.
\end{align}
 Third we conclude with Lemma~\ref{lem:concentration_covariance_matrix},
controlling randomness independently of $\Delta_{u}$. Choosing $\delta=\frac{e^{-2\beta d}}{32d\left(d+1\right)}$,
we get with probability at least $1-\epsilon_{4}$ that
\begin{align}
\Delta_{u}^{\top}\left(H^{n}-H\right)\Delta_{u} & \geq-\frac{e^{-2\beta d}}{32d\left(d+1\right)}\left\Vert \Delta_{u}\right\Vert _{1}^{2}\nonumber \\
 & \geq-\frac{e^{-2\beta d}}{2\left(d+1\right)}\left\Vert \Delta_{u}\right\Vert _{2}^{2},
\end{align}
whenever $n\geq\frac{2}{\delta^{2}}\ln\frac{p^{2}}{\epsilon_{4}}=2^{11}d^{2}\left(d+1\right)^{2}e^{4\beta d}\ln\frac{p^{2}}{\epsilon_{4}}$.
\end{proof}

\subsection{Proof of the main Theorems \label{subsec:Proof_Theorems}}
\begin{proof}[Proof of Theorem 1 (Square Error of RISE)]
We seek to apply Proposition \ref{prop:m_estimators_se} to the Regularized
Interaction Screening Estimator \eqref{eq:mr_RISE}. Using $\epsilon_{3}=\frac{2\epsilon_{1}}{3}$
in Lemma \ref{lem:gradient_concentration} and letting $\lambda=4\sqrt{\frac{\ln{3p/\epsilon_{1}}}{n}}$,
it follows that Condition \ref{cond:gradient_penalty} is satisfied
with probability greater than $1-2\epsilon_{1}/3$, whenever $n\geq e^{2\beta d}\ln\frac{3p}{\epsilon_{1}}$.

Using $\epsilon_{4}=\epsilon_{1}/3$ in Lemma \ref{lem:restricted_strong_convexity},
and observing that $3\sqrt{d}\lambda\left(\frac{e^{-3\beta d}}{4\left(d+1\right)\left(1+2\sqrt{d}R\right)}\right)^{-1}<R,$
for $R=2/\sqrt{d}$ and $n\geq2^{14}d^{2}\left(d+1\right)^{2}e^{6\beta d}\ln\frac{3p^{2}}{\epsilon_{1}}$,
we conclude that condition \ref{cond:restricted_convexity} is satisfied
with probability greater than $1-\frac{\epsilon_{1}}{3}$. Theorem
1 then follows by using a union bound and then applying Proposition
\ref{prop:m_estimators_se}. 
\end{proof}
The proof of Theorem \ref{th:Struct_RISE} becomes an immediate application
of Theorem \ref{th:mr_SE_RISE}. 
\begin{proof}[Proof of Theorem 2 (Structure Learning of Ising Models)]
According to Theorem \ref{th:mr_SE_RISE}, one observes that, with
probability $1-\epsilon_{1}$, the minimal amount of samples required
to achieve an error of $\alpha/2$ on every coupling around a single
node is 
\begin{equation}
n\geq\max\left(d/16,\alpha^{-2}\right)2^{18}d\left(d+1\right)^{2}e^{6\beta d}\ln\frac{3p^{2}}{\epsilon_{1}}.
\end{equation}
Let us choose $\epsilon_{2}=\epsilon_{1}/p$ and use the union-bound
to ensure that the couplings at every node (thresholded by $\alpha/2$)
are simultaneously recovered with probability greater than $1-\epsilon_{2}$. 
\end{proof}

\section{Numerical Results\label{sec:Numerical_Simulations}}

We test performance of the Struct-RISE, with the strength of the $l_{1}$-regularization
parametrized by $\lambda=4\sqrt{\frac{\ln(3p^{2}/\epsilon)}{n}}$,
on Ising models over two-dimensional grid with periodic boundary conditions
(thus degree of every node in the graph is $4$). We have observed
that this topology is one of the hardest for the reconstruction problem.
We are interested to find the minimal number of samples, $n_{\text{min}}$,
such that the graph is perfectly reconstructed with probability $1-\epsilon\geq0.95$.
In our numerical experiments, we recover the value of $n_{\text{min}}$
as the minimal $n$ for which Struct-RISE outputs the perfect structure
45 times from 45 different trials with $n$ samples, thus guaranteeing
that the probability of perfect reconstruction is greater than $0.95$
with a statistical confidence of at least $90\%$.

We first verify the logarithmic scaling of $n_{\text{min}}$ with
respect to the number of spins $p$. The couplings are chosen uniform
and positive $\theta_{ij}^{*}=0.7$. This choice ensures that samples
generated by Glauber dynamics are i.i.d. according to (\ref{eq:mr_gibbs_measure}).
Values of $n_{\text{min}}$ for $p\in\left\{ 9,16,25,36,49,64\right\} $
are shown on the left in Figure \ref{fig:numerical_simulations}.
Empirical scaling is, $\approx1.1\times10^{5}\ln p$, which is orders
of magnitude better than the rather conservative prediction of the
theory for this model, $3.2\times10^{15}\ln p$.

We also test the exponential scaling of $n_{\text{min}}$ with respect
to the maximum coupling intensity $\beta$. The test is conducted
over two different settings both with $p=16$ spins: the ferromagnetic
case where all couplings are uniform and positive, and the spin glass
case where the sign of couplings is assigned uniformly at random.
In both cases the absolute value of the couplings, $\left|\theta_{ij}^{*}\right|$,
is uniform and equal to $\beta$. To ensure that the samples are i.i.d,
we sample directly from the exhaustive weighted list of the $2^{16}$
possible spin configurations. The structure is recovered by thresholding
the reconstructed couplings at the value $\alpha/2=\beta/2$.

Experimental values of $n_{\text{min}}$ for different values of the
maximum coupling intensity, $\beta$, are shown on the right in Fig.~\ref{fig:numerical_simulations}.
Empirically observed exponential dependence on $\beta$ is matched
best by, $\exp\left(12.8\beta\right)$, in the ferromagnetic case
and by, $\exp\left(5.6\beta\right)$, in the case of the spin glass.
Theoretical bound for $d=4$ predicts $\exp\left(24\beta\right)$.
We observe that the difference in sample complexity depends significantly
on the type of interaction. An interesting observation one can make
based on these experiments is that the case which is harder from the
sample-generating perspective is easier for learning and vice versa.

\begin{figure}
\centering{}\includegraphics[width=6.8cm]{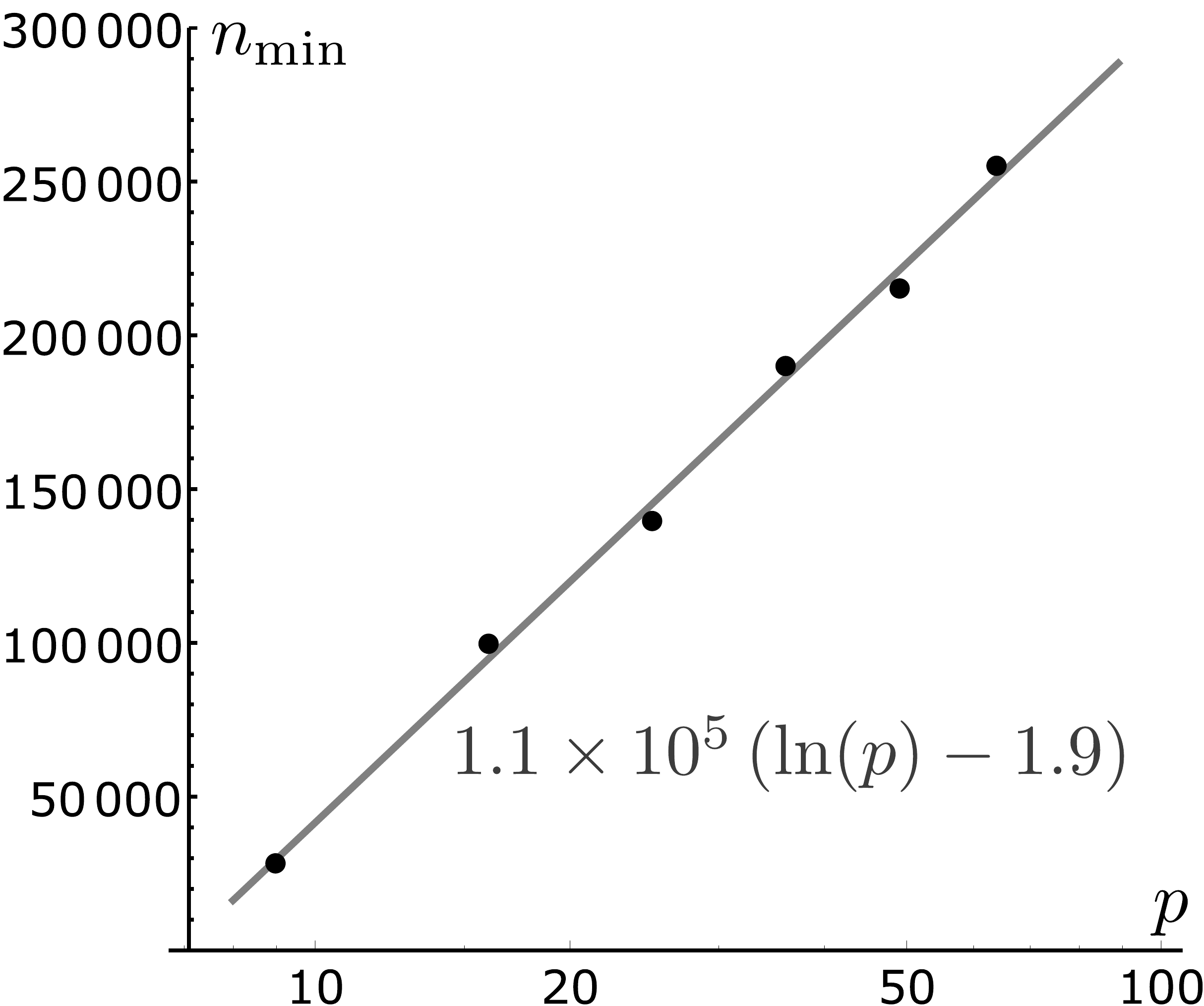}\hfill{}\includegraphics[width=6.8cm]{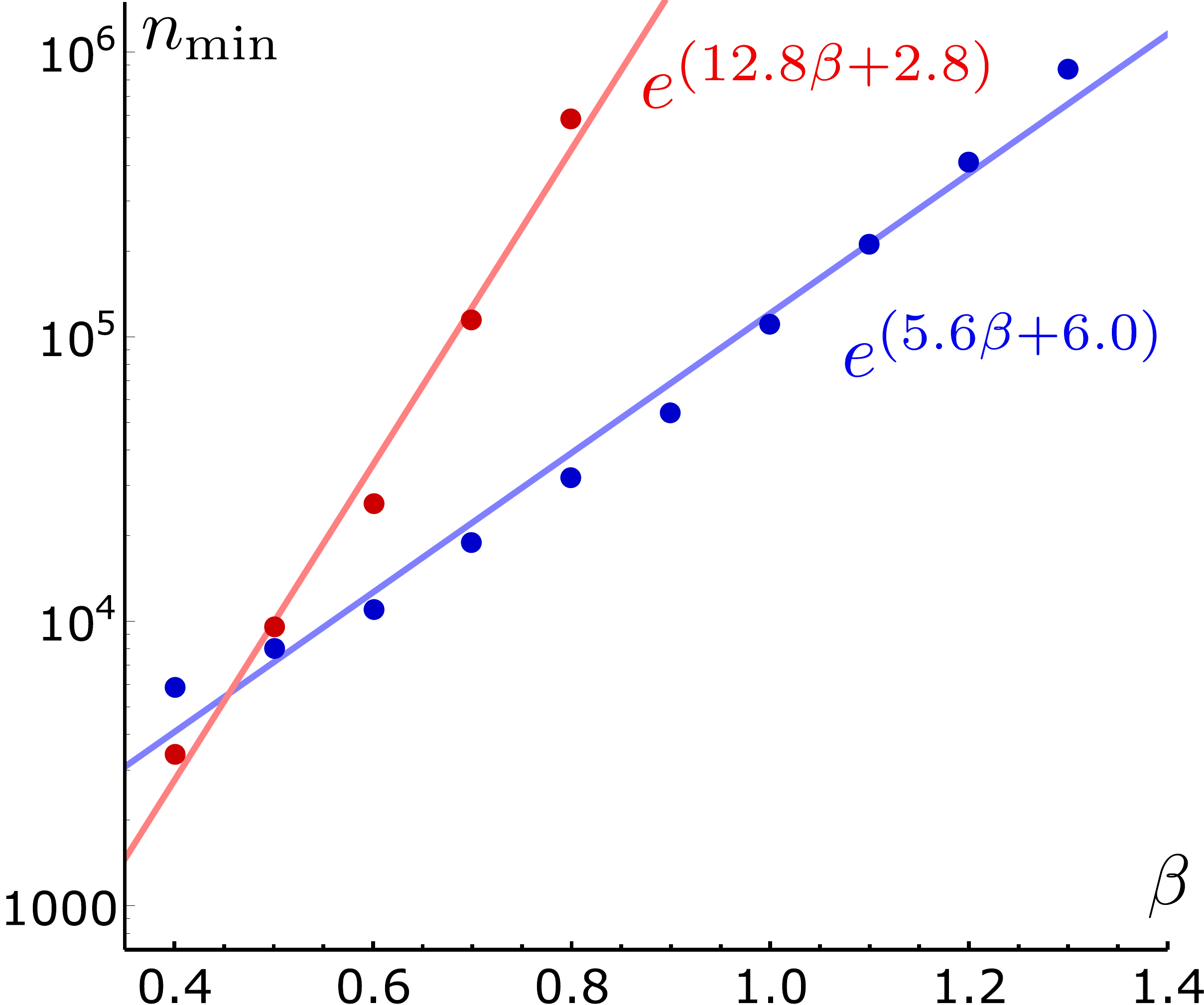}\caption{Left: Linear-exponential plot showing the observed relation between
$n_{\text{min}}$ and $p$. The graph is a $\sqrt{p}\times\sqrt{p}$
two-dimensional grid with uniform and positive couplings $\theta^{*}=0.7$.
Right: Linear-exponential plot showing the observed relation between
$n_{\text{min}}$ and $\beta$. The graph is the two-dimensional $4\times4$
grid. In red the couplings are uniform and positive and in blue the
couplings have uniform intensity but random sign.\label{fig:numerical_simulations}}
\end{figure}

\section{Conclusions and Path Forward\label{sec:Conclusion}}

In this paper we construct and analyze the Regularized Interaction
Screening Estimator (\ref{eq:mr_RISE}). We show that the estimator
is computationally efficient and needs an optimal number of samples
for learning Ising models. The RISE estimator does not require any
prior knowledge about the model parameters for implementation and
it is based on the minimization of the loss function (\ref{eq:mr_iso}),
that we call the Interaction Screening Objective. The ISO is an empirical
average (over samples) of an objective designed to screen an individual
spin/variable from its factor-graph neighbors.

Even though we focus in this paper solely on learning pair-wise binary
models, the ``interaction screening'' approach we introduce here
is generic. The approach extends to learning other Graphical Models,
including those over higher (discrete, continuous or mixed) alphabets
and involving high-order (beyond pair-wise) interactions. These generalizations
are built around the same basic idea pioneered in this paper \textendash{}
the interaction screening objective is (a) minimized over candidate
GM parameters at the actual values of the parameters we aim to learn;
and (b) it is an empirical average over samples. In the future, we
plan to explore further theoretical and experimental power, characteristics
and performance of the generalized screening estimator.

\section*{Acknowledgment}

We are thankful to Guy Bresler and Andrea Montanari for valuable discussions,
comments and insights. The work was supported by funding from the
U.S. Department of Energy's Office of Electricity as part of the DOE
Grid Modernization Initiative.

\bibliographystyle{ieeetr}
\bibliography{learning_references}

\end{document}